\documentclass{amsart}
\usepackage[foot]{amsaddr}
\usepackage{amsthm}
\usepackage{thmtools}
\usepackage{thm-restate}
\usepackage{url}            
\usepackage{booktabs}       
\usepackage{amsmath, mathtools}
\usepackage{amsfonts}       
\usepackage{nicefrac}       
\usepackage{adjustbox}
\usepackage{microtype}    
\usepackage{fancyhdr}
\usepackage{xspace,mfirstuc,tabulary}
\usepackage{amssymb}
\usepackage{bm}
\usepackage{subfig}
\usepackage{bbm}
\usepackage{supertabular}
\usepackage{algorithm}
\usepackage{algorithmic}
\usepackage{array,multirow,makecell} 
\usepackage{multicol}
\usepackage{todonotes}
\usepackage{etoolbox}
\usepackage{pgfplots}
\usepackage{tikz}
\usepackage{array,multirow,makecell}
\usepackage{scalerel,stackengine}
\usetikzlibrary{calc, intersections}
\usetikzlibrary{arrows, shapes, fit}
\usepackage{array}
\usetikzlibrary{arrows.meta}
\usepackage{hyperref}       

\newcommand{\bb}{\mathbb}
\newcommand{\R}{\bb R}

\newcommand{\N}{{\bb N}}

\newcommand{\LT}{\textup{LT}_{n}(k)}
\newcommand{\LTw}{\textup{LT}^{w}_{n}(k)}

\DeclareMathOperator{\comp}{comp}
\DeclareMathOperator{\full}{full}
\DeclareMathOperator{\PWC}{PWC}
\DeclareMathOperator{\PWL}{PWL}
\DeclareMathOperator{\CPWL}{CPWL}
\DeclareMathOperator{\SLT}{SLT}
\DeclareMathOperator{\ReLU}{ReLU}

\newtheorem{theorem}{Theorem}

\newtheorem{definition}{Definition}
\newtheorem{proposition}{Proposition}
\newtheorem{corollary}{Corollary}
\newtheorem{lemma}{Lemma}
\newtheorem{example}{Example}
\newtheorem{conjecture}{Conjecture}
\newtheorem{remark}{Remark}

\begin{document}
\title[Neural networks with linear threshold activations]{Neural networks with linear threshold activations: structure and algorithms}

\author{Sammy Khalife}
\address{Johns Hopkins University, Department of Applied Mathematics and Statistics}
\curraddr{}
\email{khalife.sammmy@jhu.edu}

\author{Hongyu Cheng}
\address{}
\curraddr{}
\email{hongyucheng@jhu.edu}

\author{Amitabh Basu}
\address{}
\curraddr{}
\email{basu.amitabh@jhu.edu}
\maketitle

\begin{abstract} 
In this article we present new results on neural networks with linear threshold 
activation functions $x \mapsto \mathbbm{1}_{\{ x > 0\}}$.
We precisely characterize the class of functions that are representable by such neural networks and show that 2 hidden layers are necessary and sufficient to represent any function representable in the class. This is a surprising result in the light of recent exact representability investigations for neural networks using other popular activation functions like rectified linear units (ReLU). We also give upper and lower bounds on the sizes of the neural networks required to represent any function in the class. Finally, we design an algorithm to solve the {\em empirical risk minimization (ERM)} problem to global optimality for these neural networks with a fixed architecture. The algorithm's running time is polynomial in the size of the data sample, if the input dimension and the size of the network architecture are considered fixed constants. The algorithm is unique in the sense that it works for any architecture with any number of layers, whereas previous polynomial time globally optimal algorithms work only for restricted classes of architectures. Using these insights, we propose a new class of neural networks that we call {\em shortcut linear threshold neural networks}. To the best of our knowledge, this way of designing neural networks has not been explored before in the literature. We show that these neural networks have several desirable theoretical properties.
\end{abstract}





%
%
%
%
%
%

\section{Introduction}
A basic question in a rigorous study of neural networks is a precise characterization of the class of functions representable by neural networks with a certain activation function. 
The question is of fundamental importance because neural network functions are a popular hypothesis class in machine learning and artificial intelligence. Every aspect of learning using neural networks benefits from a better understanding of the function class: from the statistical aspect of understanding the {\em bias} introduced in the learning procedure by using a particular neural hypothesis class, to the algorithmic question of training, i.e., finding the ``best" function in the class that extrapolates the given sample of data points. 

It may seem that the universal approximation theorems for neural networks render this question less relevant, especially since these results apply to a broad class of activation functions~\cite{hornik1991approximation,cybenko1989approximation,anthony1999neural}. We wish to argue otherwise. Knowledge of the finer structure of the function class obtained by using a particular activation function can be exploited advantageously. For example, the choice of a certain activation function may lead to much smaller networks that achieve the same bias compared to the hypothesis class given by another activation function, even though the universal approximation theorems guarantee that asymptotically both activation functions achieve arbitrarily small bias. As another example, one can design targeted training algorithms for neural networks with a particular activation function if the structure of the function class is better understood, as opposed to using a generic algorithm like some variant of (stochastic) gradient descent. This has recently led to globally optimal empirical risk minimization algorithms for {\em rectified linear units (ReLU)} neural networks with specific architecture~\cite{arora2018understanding,boob2020complexity,dey2020approximation} that are very different in nature from conventional approaches like (stochastic) gradient descent; see also~\cite{goel2017reliably,goel2018learning,goel2019learning,GKMR21,manurangsi2018computational,froese2021computational,goel2020tight}.

 
Recent results of this nature have been obtained with ReLU neural networks. Any neural network with ReLU activations clearly gives a piecewise linear function. Conversely, any piecewise linear function $\mathbb{R}^{n} \rightarrow \mathbb{R}$ can be \textit{exactly} represented with at most $\lceil \log_{2}(n+1)  \rceil$ hidden layers \cite{arora2018understanding}, thus characterizing the function class representable using ReLU activations. However, it remains an open question if $\lceil \log_{2}(n+1)  \rceil$ are indeed needed. It is conceivable that all piecewise linear functions can be represented by 2 or 3 hidden layers. It is believed this is not the case and there is a strict hierarchy starting from 1 hidden layer, all the way to $\lceil \log_{2}(n+1)  \rceil$ hidden layers. It is known that there are functions representable using 2 hidden layers that cannot be represented with a single hidden layer, but even the 2 versus 3 hidden layer question remains open. Some partial progress on this question can be found in~\cite{hertrich2021towards}.

In this paper, we study the class of functions representable using {\em threshold activations} (also known as the Heaviside activation, unit step activation, and McCulloch-Pitts neurons).
It is easy to see that any function represented by such a neural network is a piecewise constant function. We show that {\em any} piecewise constant function can be represented by such a neural network, and surprisingly -- contrary to what is believed to be true for ReLU activations -- there is always a neural network with at most 2 hidden layers that does the job. We also establish that there are functions that cannot be represented by a single hidden layer and thus one cannot do better than 2 hidden layers in general. Our constructions also show that the size of the neural network is polynomial in the number of ``pieces" of the function, in the case of fixed input dimension, similar to recent results for ReLU activations which give a polynomial size network~\cite{hertrich2021towards}. However, the degree of the polynomial in our results is linear in the dimension, compared to a quadratic dependence in the ReLU results. We also have tighter lower bounds on the size, compared to known results for ReLU networks. Moreover, we show that if we are allowed to ignore zero measure sets, the size bounds are only quadratic in the number of ``pieces", even for varying input dimension. Finally, we use these insights to design an algorithm to solve the empirical risk minimization (training) problem for these neural networks to global optimality whose running time is polynomial in the size of the data sample, assuming the input dimension and the network architecture are fixed. To the best of our knowledge, this is the first globally optimal training algorithm for any family of neural networks that works for arbitrary architectures and has computational complexity that is polynomial in the number of data points, that does not involve a discretization of parameter space or the input space. 

Linear threshold activations only retain the sign from the input (after applying an affine linear function). We now show a way to reintegrate additional input information to enhance the expressivity of the linear threshold neural networks, while maintaining  similar network sizes. For this purpose, we introduce a novel type of neural network named {\em shortcut linear threshold neural networks}. These networks are distinguished by a shortcut connection that performs a linear transform on the input, and takes the inner product with the output of the last hidden layer (see Figure \ref{fig:shortcut}). This structure enables a shift from piecewise constant to piecewise linear functions, without necessitating a change in the network's size. The novelty resides in the coupling of the network's input with the output of the last hidden layer through an inner product operation. This could be a potentially new way to design and apply neural networks which does not seem to have been explored in the literature before, to the best of our knowledge.

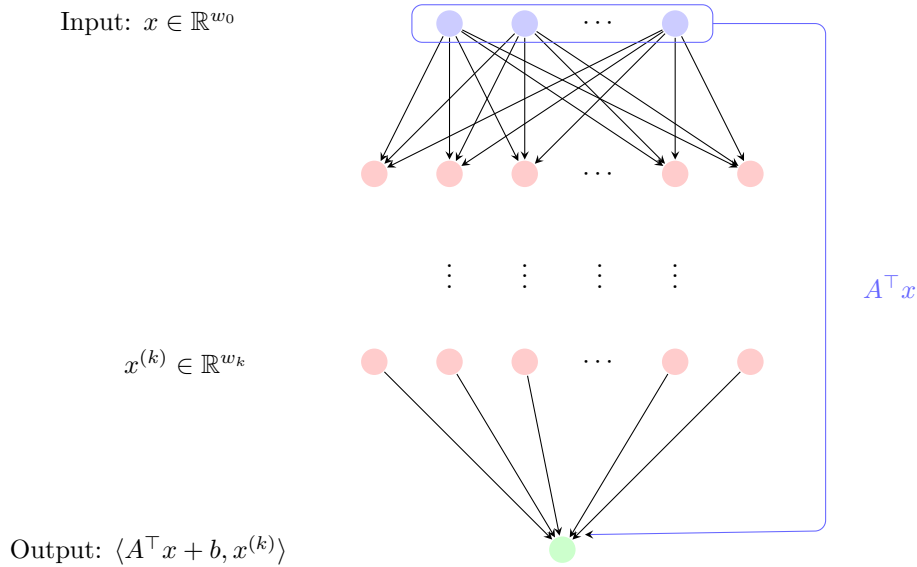
\begin{figure}[htbp] 
    \centering
	\begin{tikzpicture}
		[InputNode/.style={circle,fill=blue!20,thick, minimum size=10pt},
		HiddenNode/.style={circle,fill=red!20,thick, minimum size=10pt},
        OutputNode/.style={circle,fill=green!20,thick, minimum size=10pt}]

		\node[]()at(-1,0){Input: $x \in \mathbb{R}^{w_0}$};
		\node[]()at(-0.5,-4.5){$x^{(k)} \in \mathbb{R}^{w_k}$};
		\node[]()at(-1,-7){Output: $\langle A^\top x + b , x^{(k)} \rangle$};

		\foreach \i in {3,...,4}  
		\node[InputNode] (a_\i) at (\i,0){};  
        \node[](cdots) at (5,0) {$\cdots$};
        \node[InputNode] (a_5) at (6,0){};

		\foreach \i in {2,...,4}
		\node[HiddenNode] (b_\i) at (\i,-2){};
		\node[](cdots) at (5,-2) {$\cdots$};
		\node[HiddenNode] (b_5) at (6,-2){};
		\node[HiddenNode] (b_6) at (7,-2){};

		\foreach \i in {3,...,5}
			{\foreach \j in {2,...,6}
			\draw[-stealth] (a_\i)--(b_\j);}

		\node[](vdots) at (3,-3.25) {$\vdots$};
		\node[](vdots) at (4,-3.25) {$\vdots$};
		\node[](vdots) at (5,-3.25) {$\vdots$};
		\node[](vdots) at (6,-3.25) {$\vdots$};

		\foreach \i in {2,...,4}
			\node[HiddenNode] (c_\i) at (\i,-4.5){};
			\node[](cdots) at (5,-4.5) {$\cdots$};
			\node[HiddenNode] (c_5) at (6,-4.5){};
			\node[HiddenNode] (c_6) at (7,-4.5){};
		
		\node[OutputNode] (output) at (4.5,-7){};

		\foreach \i in {2,...,6}
		\draw[-stealth] (c_\i)--(output);
		
		\draw[rounded corners, draw=blue!60] (2.5,0.25) rectangle (6.5,-0.25);
		\draw[-stealth,draw=blue!60,rounded corners] (6.5,0)--(8,0)--(8,-6.75)--(4.8,-6.8);
		\node[]() at (8.5,-3.5) {\textcolor{blue!60}{$\qquad A^\top x$}};
	\end{tikzpicture}
	\caption{Illustration of a $\R^{w_0} \rightarrow \R$ shortcut linear threshold neural networks with $k$ hidden layers, where $x \in \R^{w_0}$ is the input to the network, $x^{(k)} \in \R^{w_k}$ is the output of the $k$-th hidden layer. }
	\label{fig:shortcut}
\end{figure}

These shortcut linear threshold networks can represent piecewise linear functions that are possibly discontinuous. This is a strict superset of the family of functions representable by ReLU neural networks that give {\em continuous} piecewise linear functions. Nevertheless, we show that the model complexity of this new network is not significantly larger than ReLU networks  and
these new neural networks can be trained provably to global optimality using the ERM algorithm we develop for linear threshold activations. For a more comprehensive discussion on this topic, we direct readers to Section~\ref{sec:5.1.2}. These results provide some evidence that shortcut linear threshold networks could be a superior class compared to ReLU networks. While the results we present are all theoretical in nature, we believe they provide reasonable motivation to explore the potential of this new class of neural networks in applications. We leave this avenue open for future work.



\section{Formal statement of results}

We first introduce necessary definitions and notation, followed by a precise statement of our results.

\subsection{Definitions and notations}

\subsubsection{Polyhedral theory}
\begin{definition} A {\em polyhedral complex} $\mathcal{P}$ is a collection of polyhedra having the following properties: 
\begin{itemize}
    \item[(A)] For every $P, P' \in \mathcal{P}, P \cap P' $ is a common face of $P$ and $P'$.
    \item[(B)] every face of a polyhedron in $\mathcal{P}$ belongs to $\mathcal{P}$. 
\end{itemize}
\end{definition}




We denote by $\dim(P)$ the dimension of a polyhedron and by $\mathring{P}$ the relative interior of $P$. $|\mathcal{P}|$ will denote the number of polyhedra in a polyhedral complex $\mathcal{P}$ and is called the {\em size of $\mathcal{P}$}. The set of full-dimensional polyhedra in $\mathcal{P}$ is denoted by $\full(\mathcal{P})$, and thus, $|\full(\mathcal{P})|$ corresponds to the number of full-dimensional polyhedra in $\mathcal{P}$. 


\begin{definition}[Piecewise linear and piecewise constant functions]
    We say that a function $f: \R^n \rightarrow \R$ is {piecewise linear} if there exists a finite polyhedral complex that covers $\R^n$ and $f$ is affine linear in the relative interior of each polyhedron in the complex. If each of the affine functions are constant functions, i.e., $f$ is constant in the relative interior of each polyhedron, then we call such a function {\em piecewise constant}. We use PWL$_n$ and $\text{PWC}_n$ to denote the class of all piecewise linear functions and piecewise constant functions (respectively) from $\R^n$ to $\R$; thus, $\text{PWC}_n\subseteq \text{PWL}_n$.  We will also use CPWL$_n$ to denote subclass of piecewise linear functions from $\R^n$ to $\R$ that are also continuous.
\end{definition}


\begin{definition}
Let $f \in \textup{PWL}_n$. We say that $x \in \mathbb{R}^{n}$ is a {\em regular point} for $f$ if there exists $\epsilon > 0$ such that $f$ is affine linear on the ball centered at $x$ and radius $\epsilon$. A point that is not regular is called a {\em breakpoint} for $f$.
\end{definition}

Note that there may be multiple polyhedral complexes that correspond to a given piecewise linear or constant function, with possibly different sizes. For example, the indicator function of the nonnegative orthant $\R^n_+$ is a piecewise constant function but there are many different ways to break up the complement of the nonnegative orthant into polyhedral regions. This motivates the following definitions.

\begin{definition} We say that a polyhedral complex {\em $\mathcal{P}$ is compatible with a piecewise linear or constant function $f$} if $f$ is linear or constant (respectively) in the relative interior of every polyedron in $\mathcal{P}$. Moreover, for any function $f\in \text{PWL}_n$,  $\comp(f)$ refers to the set of all polyhedral complexes compatible with $f$. We denote $\mathcal{P}^*_f$ as a polyhedral complex with the smallest cardinality in $\comp(f)$, that is, $\mathcal{P}^*_f \in \arg\min_{\mathcal{P} \in \comp(f)}|\mathcal{P}|$. Consequently, {\em the number of pieces of $f$}, denoted as $|f|$, is defined as $|\full(\mathcal{P}^*_f)|$, which corresponds to the number of full-dimensional polyhedra in $\mathcal{P}^*_f$.
\end{definition}



{\subsubsection{Neural network terminology}}
\begin{definition}[Neural networks (NN)]\label{def:DNN}  Fix an {\em activation function} $\sigma: \R \to \R$. For any number of hidden layers $k \in \mathbb{N}$, input and output dimensions $w_0$, $w_{k+1} \in \mathbb{N}$, a $\mathbb{R}^{w_0} \rightarrow \mathbb{R}^{w_{k+1}}$ {\em neural network (NN) with $\sigma$ activation} is given by specifying a sequence of $k$ natural numbers $w_1, w_2, \cdots, w_k$ representing widths of the hidden layers and a set of $k+1$ affine transformations $T_i:
\mathbb{R}^{w_{i-1}} \rightarrow \mathbb{R}^{w_{i}}$, $i=1, \ldots, k+1$. Such a NN is called a $(k+1)$-layer NN, and is said to have $k$ hidden layers. The function $f: \mathbb{R}^{w_0} \rightarrow \mathbb{R}^{w_{k+1}}$ computed or represented by this NN is:
$$f= T_{k+1}\circ \sigma \circ T_{k} \circ \cdots  T_2 \circ \sigma \circ T_1.$$ 
If $T_i$ is represented by the matrix $A^i \in \R^{w_i\times w_{i-1}}$ and vector $b^i \in \R^{w_i}$, i.e., $T_i(x) = A^ix + b^i$, then the {\em weights of neuron} $j \in \{1, \ldots, w_{i}\}$ in the $i$-th hidden layer are given by the entries of the $j$-th row of $A^i$ and the {\em bias} of this neuron is given by the $j$-th coordinate of $b^i$. The set of all weights and biases of all neurons is called the set of {\em learning parameters} of the NN, and the {\em size} of the NN, or the number of neurons in the NN, is $w_1 + \cdots + w_k $.
\end{definition}

\begin{definition} The {\em threshold activation function} is a map from $\R$ to $\{0,1\}$ given by the indicator of the positive reals, i.e., $x > 0$. By extending this to apply coordinatewise, we get a function $\sigma: \R^d \to \{0,1\}^d$ for any $d \geq 1$, i.e., $\sigma(x)_i$ is 1 if and only if $x_i > 0$ for $i=1, \ldots, d$. For any subset $X \subseteq \R^n$, $\mathbbm{1}_X$ will denote its {\em indicator function}, i.e., $\mathbbm{1}_X(y) = 1$ if $y \in X$ and $0$ otherwise.
\end{definition}

\begin{definition}[Threshold and ReLU activations]\label{def:linear-threshold-ReLU-DNNs}  {\em Linear threshold neural networks} are those NNs where $\sigma$ is the threshold activation function defined above. For natural numbers $n,k$ and a tuple $w = (w_1, \ldots, w_k)$, we use $\LTw$ to denote the family of all possible linear threshold NNs with input dimension $w_0 = n$, $k$ hidden layers with widths $w_1, \ldots, w_k$ and output dimension $w_{k+1} = 1$. $\LT:= \bigcup_{w=(w_1, \ldots, w_k)} \LTw $ will denote the family of all linear threshold activation neural networks with $k$ hidden layers.

{\em ReLU neural networks} are those NNs where $\sigma(x) = \max\{0,x\}$, which is called the {\em Rectified Linear Unit (ReLU)} activation function. Analogous to the notation for linear threshold functions, we introduce ReLU$_n^w(k)$ and ReLU$_n(k)$ for ReLU neural networks.
\end{definition}

We rigorously define Shortcut Linear Threshold Neural Networks (SLT NNs) as follows: Given a $\R^{w_0} \rightarrow \R$ linear threshold NN with $k \in \mathbb{N}$ hidden layers and input $x \in \R^{w_0}$, the network's output is constructed as a linear combination of the outputs of the neurons in the final hidden layer. Specifically, the output equals $\langle b , x^{(k)} \rangle$, where $x^{(k)} = [\mathbbm{1}_{X_1}(x),\dots,\mathbbm{1}_{X_{w_k}}(x)]^\top$ represents the output of the $k-$th layer, and $b \in \R^{w_k}$. Distinctively, instead of employing a constant vector $b$ for the ultimate linear combination, we utilize an affine linear transformation of the original input $x$ as the linear coefficients. Hence, the output of the shortcut linear threshold NN is defined as $\langle A^\top x+b , x^{(k)} \rangle$, where $A \in \R^{w_0\times w_k}$ and $b \in \R^{w_k}$. Notably, by setting $A=0$, we revert to linear threshold NNs as defined in Definition~\ref{def:linear-threshold-ReLU-DNNs}.

Analogously to linear threshold and ReLU NNs, we denote the class of functions represented by shortcut linear threshold NNs with $w_0 = n$, $w_{k+1}=1$, $k$ hidden layers, and $w=(w_1, \ldots, w_k)$ signifying the widths of the hidden layers as SLT$_n^w(k)$ and SLT$_n(k)$.

\begin{definition}[Shortcut linear threshold NNs]\label{def:SLT} We define shortcut linear threshold NNs as a type of linear threshold NNs with a shortcut connection. More formally, consider a linear threshold NN with $k \in \mathbb{N}$ hidden layers, an input vector $x \in \R^{w_0}$, and an output of the $k$-th hidden layer, denoted as $x^{(k)} \in \R^{w_k}$. For a shortcut linear threshold NN based on this linear threshold NN, the output is defined as $\langle A^\top x+b , x^{(k)} \rangle$, where $A \in \R^{w_0\times w_k}$ and $b \in \R^{w_k}$. It is worth noting that choosing $A=0$ recovers linear threshold NNs as defined in Definition~\ref{def:linear-threshold-ReLU-DNNs}.

    Analogous to linear threshold and ReLU NNs, we use SLT$_n^w(k)$ and SLT$_n(k)$ to denote the class of functions represented by shortcut linear threshold NNs with $w_0 = n$, $w_{k+1}=1$, $k$ hidden layers, and $w=(w_1, \ldots, w_k)$ representing the widths of the hidden layers. 
\end{definition}

    Note that the novelty in the above definition is in how the output is derived from the output of the final hidden layer and the original input. To the best of our knowledge, such a modification in the definition of neural representations of functions is new. We present some results that we feel give evidence for the usefulness of this definition.

To better present our results regarding the size bounds of neural networks that are capable of computing a specific function, we introduce the notation $\mathcal{N}(H,f)$ to represent the set of all neural networks in $H$ that can compute a given function $f$. For instance, $\mathcal{N}(\LT,f)$ represents the set of all linear threshold neural networks with an input dimension of $n$ and $k$ hidden layers that can compute the function $f$. We further use $\mathcal{N}_\mu(H,f)$ to denote the set of all neural networks in $H$ that can compute a function $g$ satisfying $\mu(\{x:f(x) \neq g(x)\}) = 0$, where $\mu(\cdot)$ denotes the Lebesgue measure. The size of a neural network $N$ is denoted by $|N|$.


\subsection{Our contributions}

\subsubsection{Results for linear threshold NNs}

Any function expressed by a linear threshold neural network is a constant piecewise function (i.e. $\LT \subseteq \text{PWC}_{n}$ for all natural numbers $k$), because a composition of piecewise constant functions is piecewise constant. In this work we show that linear threshold neural networks with 2 hidden layers can compute any constant piecewise function, i.e. $\text{LT}_n(2) = \text{PWC}_{n} $. We also prove that this is optimal, in the sense that a single hidden layer does not suffice to represent all piecewise constant functions. More formally,


\begin{restatable}{theorem}{theoremone}
\label{theorem:1} $\textup{LT}_1(1) = \PWC_1$, and for all natural numbers $n, k \geq 2$, $$\textup{LT}_n(1) \subsetneq \textup{LT}_n(2) = \LT = \PWC_n.$$ 
Equivalently, any piecewise constant function $f: \mathbb{R}^{n}\rightarrow \mathbb{R}$ can be computed by a linear threshold NN with at most $2$ hidden layers. Moreover,
\begin{equation*}
   \min_{N \in \mathcal{N}(\textup{LT}_n(2),f)} |N| \leq 3 |\mathcal{P}^*_f|.
\end{equation*} 
\end{restatable}

Next, we show that the bound on the size of the neural network in Theorem~\ref{theorem:1} is in a sense best possible, up to constant factors. 

\begin{restatable}{proposition}{proplowerboundclass}
\label{PROPOSITION:LOWER_BOUND_CLASS}
There exists a family of functions $f_n \in \PWC_n$ such that 
\begin{equation*}
    {\min_{N \in \mathcal{N}( \textup{LT}_n(2), f_n)} |N| \geq |\mathcal{P}_{f_n}^*|,\ \forall n \in \mathbb{N}_+.}
\end{equation*}

\end{restatable}


Notwithstanding Proposition~\ref{PROPOSITION:LOWER_BOUND_CLASS}, there still remains the possibility that the construction we give to prove the equalities in Theorem~\ref{theorem:1} is suboptimal in terms of the size of the networks produced by our construction. Going beyond our specific construction, it is {\em a priori} possible that there are families of piecewise constant functions that can be represented with polynomial size circuits if one uses more than 2 hidden layers, while any linear threshold NN with 2 hidden layers has super polynomial size. Such results have been established for NNs involving ReLU and other activation functions; see, e.g.,~\cite{arora2018understanding,telgarsky2016benefits,eldan2016power,pmlr-v49-cohen16} for a representative sample. 


Building upon our previous discussion and results that also consider polyhedra of lower dimensions, 
 it is important to highlight that the main focus in practical applications is full-dimensional polyhedra. This arises from the fact that for any countable set of data points, there is zero probability of them being contained within lower-dimensional polyhedra. 
 Furthermore, this approach provides an equitable basis for comparing against size results for ReLU NNs, which represent continuous functions whose values on lower-dimensional polyhedra are determined by the values on full-dimensional polyhedra. Thus, we formulate the ensuing theorem related to the smallest linear threshold NN size expressing a given function, placing our focus solely on full-dimensional polyhedra.

\begin{theorem} \label{th:a.e.SLTn}
    Let $f\in$ $\PWC_n$ with $p \geq 2$ pieces. Then, 
    \begin{equation*}
        \min_{N \in \mathcal{N}_\mu(\textup{LT}_n(2),f)} |N| \leq \frac{p(p+1)}{2}+1.
    \end{equation*}
\end{theorem}
The applicability and significance of this theorem will become more evident in Section~\ref{sec:SLTNN}, specifically in the context of our proposed network structure. When expressing a given continuous piecewise linear function, this theorem provides an upper bound for the minimum neural network size that is significantly smaller compared to existing ReLU NN results. However, Theorem~\ref{th:a.e.SLTn} ignores sets of measure zero to derive an upper bound quadratic in $p$. If we seek to precisely express a given function $f \in \PWC_n$ with $p$ pieces, the bound increases to $\mathcal{O}(p^{n+1})$ as shown in the following proposition.

\begin{restatable}{proposition}{propupperbound} \label{th:upperboundSLTn}
    Let $f\in$ {$\PWC_n$} with $p \geq n+1$ pieces. Then,  
    \begin{equation*}
        \min_{N \in \mathcal{N}(\textup{LT}_n(2),f)} |N| \leq 3 \left( \frac{ep}{n+1} \right)^{n+1}.
    \end{equation*}
\end{restatable}


\subsubsection{Algorithm for exact empirical risk minimization }

{ In addition to our structural results,} we present a new algorithm to perform exact {\em empirical risk minimization (ERM)} for linear threshold neural networks with fixed architecture, i.e., fixed $k$ and $w=(w_1, \ldots, w_k)$. Given $D$ data points $(x_i, y_i) \in \mathbb{R}^{n} \times \mathbb{R}, \, i=1, \cdots, D$, the ERM problem with hypothesis class $\LTw$ is 

\begin{equation}\label{eq:problem_3_layer_HDNN}
\min_{f\in \LTw } \; \frac{1}{D} \sum_{i=1}^{D} \ell(f(x_i), y_i), 
\end{equation}
where $\ell$ is a convex loss function. 

\begin{restatable}{theorem}{thmoregeneralcomplexity}
\label{th:more_general_complexity}
 For natural numbers $n, k$ and tuple $w = (w_1, \ldots, w_k)$, there exists an algorithm that computes the global optimum~\eqref{eq:problem_3_layer_HDNN}, up to $\epsilon$-accuracy, with running time 
  $O(D^{w_1 n}\cdot 2^{\sum_{i=1}^{k-1}w_i^{2} w_{i+1}}\cdot  \textup{poly}(D, w_1, \ldots, w_k, \log\left(\frac{1}{\epsilon}\right)))$. If $\ell$ is the absolute value difference, then the global optimum can be computed exactly.
\end{restatable}
Thus, the algorithm is polynomial in the size of the data sample, if $n$, $k$, $w_1$, $\ldots$, $w_k$ are considered fixed constants.





\subsubsection{Shortcut linear threshold NNs}\label{sec:SLTNN}


\begin{theorem}\label{th:SLTnrepresentability}
    \textup{ReLU}$_n{(\lceil \log_2(n+1) \rceil)}=\textup{CPWL}_n$ $\subsetneq$ \textup{PWL}$_n=\textup{SLT}_n(2)$.
\end{theorem}
Moreover, for a given function $f \in \PWL_n$, we can derive the same upper bound on the size of shortcut linear threshold neural networks as in {Proposition~\ref{th:upperboundSLTn} and Theorem~\ref{th:a.e.SLTn}}, using analogous arguments. {Furthermore, for representing continuous piecewise linear functions, the continuity of the function allows us to modify our construction to eliminate the need to ignore sets of zero measure while asserting an upper bound quadratic in $p$. This particular insight contributes to the subsequent corollary of Theorem~\ref{th:a.e.SLTn}.
\begin{corollary}\label{cor:a.e.SLT}
    Let $f\in$ $\CPWL_n$ with $p \geq 2$ pieces. Then, 
    \begin{equation*}
        \min_{N \in \mathcal{N}(\textup{SLT}_n(2),f)} |N| \leq p^2+1.
    \end{equation*}
\end{corollary}
}

{We can compare Corollary~\ref{cor:a.e.SLT} with known bounds on the sizes of ReLU NNs. For a fixed CPWL$_n$ function $f$ with $p$ pieces, Theorem~2.1 in \cite{arora2018understanding} establishes that a ReLU NN needs no more than $\lceil \log_2(n+1) \rceil$ hidden layers to compute $f$. In contrast, our SLT$_n$ construction described in Theorem~\ref{th:SLTnrepresentability} requires only $2$ hidden layers. Additionally, Theorem~4.4 in \cite{hertrich2021towards} states that the ReLU NN, with $\lceil \log_2(n+1) \rceil$ hidden layers to compute $f$, will have width bounded by $\mathcal{O}(p^{2n^2+3n+1})$. In comparison, our Corollary~\ref{cor:a.e.SLT} implies a significantly tighter bound for the size of the SLT$_n$ network to compute a same function $f$, namely $\mathcal{O}(p^{2})$.}


It is also straightforward to modify the ERM algorithm presented in Theorem~\ref{th:more_general_complexity} to apply to shortcut linear threshold NNs with the same architecture. {We direct readers to Section~\ref{sec:ERMforSLT} for details.}

\bigskip
The rest of the article is organized as follows. Section~\ref{sec:prelim} collects some general structural results on neutral networks that use threshold activations and introduces some concepts useful for this analysis that will be used throughout the paper. Section~\ref{sec:2} gives the proofs of the results stated above for linear threshold NNs. Section~\ref{sec:ShortcutNN} provides the proofs of the results involving shortcut linear threshold NNs discussed above. Section~\ref{sec:5} closes the paper with a discussion of some open questions.


\section{Preliminary results for threshold activations}\label{sec:prelim}

In this section, we will collect some structural results for neural networks with linear threshold activations. These will be useful for the proofs of our main results.


\begin{definition}
Let $m \geq 1$ be any natural number. We say that a collection $\mathcal{A}$ of subsets of $\{1, \ldots, m\}$ is {\em linearly separable} if there exist $\alpha_1, \ldots, \alpha_m, \beta \in \R$ such that any subset $A \subseteq \{1, \ldots, m\}$ is in $\mathcal{A}$ if and only if $\sum_{s\in A} \alpha_s + \beta > 0$. {$\mathcal{L}_m$ refers to the set of all linearly separable collections of subsets of $\{1, \ldots, m\}$.}
\end{definition}

\begin{remark} We note that given a collection $\mathcal{A}$ of subsets of  $\{1, \cdots, m\}$ 
one can test if $\mathcal{A}$ is linearly separable by checking if the optimum value of the following linear program is strictly positive:
\begin{align*}
    &\max_{t\in \mathbb{R}, \alpha \in \mathbb{R}^{m}, \beta \in \mathbb{R}} \;\; t 
    \\
    \textup{s.t.} \;\; \sum_{s\in A} \alpha_s + \beta \geq t \quad & \forall A \in \mathcal{A} \quad  \textup{and} \quad
     \sum_{s\in A} \alpha_s + \beta \leq 0 \quad \forall A \notin \mathcal{A}
\end{align*}
In Algorithm~\ref{Algorithm:heavisideopt} below, we will enumerate through all possible collections in $\mathcal{L}_m$ (for different values of $m$). We assume this has been done a priori using the linear programs above and this enumeration can be done in time $|\mathcal{L}_m|$ during the execution of Algorithm~\ref{Algorithm:heavisideopt}.
\end{remark}

\begin{example}
  In $\mathbb{R}^{2}$, $\mathcal{A}=\{\emptyset, \{1\}, \{2\}\}$ is linearly separable, but $\{\emptyset, \{1,2\}\}$ is not linearly separable because the set of inequalities $\beta > 0$, $\alpha_1 + \alpha_2 + \beta > 0$, $\alpha_1 + \beta \leq 0 $ and $\alpha_2 + \beta \leq 0$ have no solution. Two examples of linearly separable collections in $\mathbb{R}^{3}$ are given in Figure \ref{fig:splitting_cubes}. 
\end{example}





\newcommand{\Depth}{2}
\newcommand{\Height}{2}
\newcommand{\Width}{2}

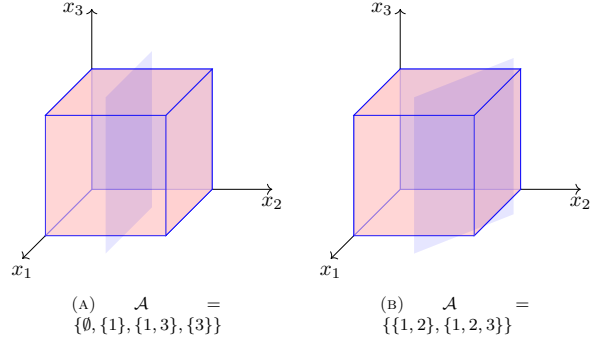
\begin{figure}

    \centering
   \scalebox{0.8}{ 
    \subfloat[$\mathcal{A}=\{ \emptyset, \{1\}, \{1,3\}, \{3\}\}$]{
   \begin{tikzpicture}
\coordinate (O) at (0,0,0);

\coordinate (A) at (0,\Width,0);

\coordinate (B) at (0,\Width,\Height);

\coordinate (C) at (0,0,\Height);

\coordinate (D) at (\Depth,0,0);

\coordinate (E) at (\Depth,\Width,0);

\coordinate (F) at (\Depth,\Width,\Height);

\coordinate (G) at (\Depth,0,\Height);

 \draw[->] (0,0,0) -- (3,0,0) node[below]{$x_2$};
 \draw[->] (0,0,0) -- (0,0,3) node[below]{$x_1$};
 \draw[->] (0,0,0) -- (0,3,0) node[left]{$x_3$};

\draw[blue,fill=red!10] (O) -- (C) -- (G) -- (D) -- cycle;
\draw[blue,fill=blue!30] (O) -- (A) -- (E) -- (D) -- cycle;
\draw[blue,fill=red!10] (O) -- (A) -- (B) -- (C) -- cycle;
\draw[blue,fill=red!20,opacity=0.8] (D) -- (E) -- (F) -- (G) -- cycle;
\draw[blue,fill=red!20,opacity=0.6] (C) -- (B) -- (F) -- (G) -- cycle;
\draw[blue,fill=red!20,opacity=0.8] (A) -- (B) -- (F) -- (E) -- cycle;

 \fill[blue,opacity=0.1] (1,-0.3,0) -- (1,-0.3,2) -- (1,2.3,2) --  (1,2.3,0) -- node[midway,right]{} cycle;
\end{tikzpicture}

    }}
 \scalebox{0.8}{\subfloat[$\mathcal{A}=\{\{1,2\}, \{1,2,3\}\}$]{
   \begin{tikzpicture}
\coordinate (O) at (0,0,0);

\coordinate (A) at (0,\Width,0);

\coordinate (B) at (0,\Width,\Height);

\coordinate (C) at (0,0,\Height);

\coordinate (D) at (\Depth,0,0);

\coordinate (E) at (\Depth,\Width,0);

\coordinate (F) at (\Depth,\Width,\Height);

\coordinate (G) at (\Depth,0,\Height);

 \draw[->] (0,0,0) -- (3,0,0) node[below]{$x_2$};
 \draw[->] (0,0,0) -- (0,0,3) node[below]{$x_1$};
 \draw[->] (0,0,0) -- (0,3,0) node[left]{$x_3$};

\draw[blue,fill=red!10] (O) -- (C) -- (G) -- (D) -- cycle;
\draw[blue,fill=blue!30] (O) -- (A) -- (E) -- (D) -- cycle;
\draw[blue,fill=red!10] (O) -- (A) -- (B) -- (C) -- cycle;
\draw[blue,fill=red!20,opacity=0.8] (D) -- (E) -- (F) -- (G) -- cycle;
\draw[blue,fill=red!20,opacity=0.6] (C) -- (B) -- (F) -- (G) -- cycle;
\draw[blue,fill=red!20,opacity=0.8] (A) -- (B) -- (F) -- (E) -- cycle;

 \fill[blue,opacity=0.1] (2,-0.3,0.3) -- (1,-0.3,2) -- (1,2.3,2) --  (2,2.3,0.3) -- node[midway,right]{} cycle;
\end{tikzpicture}

    }}
         \caption{Two linearly separable collections of $2^{\{1, 2, 3\}}$ in $\mathbb{R}^{3}$. The subsets of $\{1, 2, 3\}$ are represented by the vertices of $\{0,1\}^3$. The blue hyperplanes represent a possible separation of the corresponding vertices, giving two different linearly separable collections. }   
    \label{fig:splitting_cubes}

\end{figure}


\begin{lemma}\label{lemma:4}
Let $k\geq 2$ and $w = (w_1, \ldots, w_k)$, and consider a NN of {\textup{LT}$_n^{w}(k)$} or $\SLT_n^w(k)$. {The output of each neuron in this neural network is the indicator function of a specific subset of $\R^n$.}
Suppose we fix the weights of the neural network up to the $(k-1)$-st hidden layer. This fixes the sets $Y_1, \ldots, Y_{w_{k-1}} \subseteq \R^n$ computed by the $w_{k-1}$ neurons in this layer. 


{Then the output of a neuron in the $k$-th layer is the indicator function of $X\subseteq \R^n$} (by adjusting the weights and bias of this neuron) if and only if there exists a linearly separable collection $\mathcal{A}$ of subsets of $\{1, \ldots, w_{k-1}\}$ such that:
$$X = \bigcup_{A \in \mathcal{A}} \left[ \, \left( \, \bigcap_{s \in A} Y_s \, \right)  \, \cap \, \left( \, \bigcap_{s \notin A} Y_s^{c} \, \right) \, \right]{.}$$ 
\end{lemma}
\begin{proof}
Let $\alpha \in \mathbb{R}^{w_{k-1}}$, $\beta \in \mathbb{R}$ be the weights and bias of the neuron in the $k$-th layer.  By definition, the set represented by this neuron is
$$S_{\alpha, \beta} = \{x \in \mathbb{R}^{n}: \alpha_1 \mathbbm{1}_{ Y_1}(x) + \cdots + \alpha_{w_{k-1}}\mathbbm{1}_{ Y_{w_{k-1}}}(x) + \beta > 0\}{.}$$

We can suppose without loss of generality that $S_{\alpha, \beta}$ is non empty, otherwise the output of the neuron is always $0$ and the property is true by taking $\mathcal{A}=\emptyset$. Therefore, we define the collection $\mathcal{A} := \{ A \subseteq \{1, \cdots, w_{k-1}\}: \sum_{i \in A} \alpha_i  + \beta > 0\} $. Since $S_{\alpha, \beta}$ is non empty, $\mathcal{A}$ is non empty, and by definition, is a linearly separable collection. Now consider the set:
$$O:=\bigcup_{A \in \mathcal{A}} \left[ \, \left( \, \bigcap_{s \in A} Y_s \, \right)  \, \cap \, \left( \, \bigcap_{s \notin A} Y_s^{c} \, \right) \, \right]{.}$$
We will prove now $O=S_{\alpha, \beta}$. We first show $O \subseteq S_{\alpha, \beta}$. If $O=\emptyset$, this is trivial. Else, let $x  \in O$. Then there exists $A \in \mathcal{A}, \; A \subseteq \{1, \cdots, w_{k-1}\}$ such that $ x \in (\, \bigcap_{s \in A} Y_s \,) \, \cap \, (\, \bigcap_{s \notin A} Y_s^{c}\, ) $, and by definition of $\mathcal{A}$ $$\sum_{s \in A} \alpha_s \mathbbm{1}_{ Y_s}(x)  + \beta >  0{.}$$
The previous inequality implies that:
\begin{align*}
    \sum_{i=1}^{w_{k-1}} \alpha_k \mathbbm{1}_{ Y_k}(x)  + \beta & = \sum_{s \in A} \alpha_s \mathbbm{1}_{ Y_s}(x)   + \sum_{s \notin A} \alpha_s \mathbbm{1}_{ Y_s}(x) + \beta     \\
     &=  \sum_{s \in A} \alpha_s \mathbbm{1}_{ Y_s}(x)  + \beta >  0{.}
\end{align*}
The first equality holds because $x\notin Y_s $ if and only if  $ s\notin A$. This means that $x \in S_{\alpha, \beta}$, hence $O \subseteq S_{\alpha, \beta}$.

To show the reverse inclusion, let $x \in S_{\alpha, \beta}$. Then $ \alpha_1 \mathbbm{1}_{ Y_1}(x) + \cdots  + \alpha_{w_{k-1}}\mathbbm{1}_{ Y_{w_{k-1}}}(x)  + \beta >  0$.
Let $A := \{s: x \in Y_s\}$. Then:
$$ \sum_{s \in A } \alpha_s   + \beta = \sum_{i=1}^{w_{k-1}} \alpha_k \mathbbm{1}_{ Y_k}(x) + \beta  >   0 {,}$$

hence $A\in \mathcal{A}$, and by construction $x \in (\, \bigcap_{s \in A} Y_s ) \, \cap \, (\, \bigcap_{s \notin A} Y_s^{c}  \,) $, hence $x\in O$, and $S_{\alpha, \beta} \subseteq O$.

\medskip

Conversely, let $\mathcal{A}$ be a linearly separable collection of subsets of $\{1, \cdots, w_{k-1}\}$. By definition there exists $\alpha \in \mathbb{R}^{n}$ and $\beta \in \mathbb{R}$ such that $A \in \mathcal{A} $ if and only if $ \sum_{s \in A }  \alpha_s + \beta > 0  $. $\alpha$ and $\beta$ can be chosen as the weights of the neuron in the $k$-th hidden layer and its output is the function 
$\mathbbm{1}_{\{x\in \R^n\;:\;\sum_{i=1}^{w_{k-1}}\alpha_i \mathbbm{1}_{ Y_i}(x) + \beta > 0\}}$.
\end{proof}

The following is a corollary of Lemma~\ref{lemma:4}, which indicates that breakpoints are non-increasing as we proceed through the hidden layers.
\begin{corollary}\label{lemma:breakpointsbetweenlayers}
    Let $k\geq 2$ and $w = (w_1, \ldots, w_k)$, and consider a NN of $\LTw$ or $\SLT_n^w(k)$. Then the breakpoints of the output of any neuron in the $j$-th layer are included in the union of the breakpoints of the neurons of the $ (j-1)$-st layer, where $2 \leq j \leq k$.
\end{corollary}
\begin{proof}
    {Let $\mathbbm{1}_{Y_1},\dots,\mathbbm{1}_{Y_{w_{j-1}}}$ be the functions computed by neurons in the $(j-1)$-st layer, where $2 \leq j \leq k$. Consider any neuron in the $j$-th layer, suppose it computes $\mathbbm{1}_X$.} By Lemma~\ref{lemma:4}, there exists a linearly separable collection $\mathcal{A}$ of subsets of $\{1,\ldots,w_{j-1}\}$ such that 
    \begin{equation*}
    X = \bigcup_{A \in \mathcal{A}}\left[ \left( \bigcap_{s \in A}Y_s \right) \cap \left( \bigcap_{s \notin A}Y_s^c \right) \right],
    \end{equation*}
    then we have
    \begin{align*}
        \partial X &\subseteq \bigcup_{A \in \mathcal{A}} \partial \left[ \left( \bigcap_{s \in A}Y_s \right) \cap \left( \bigcap_{s \notin A}Y_s^c \right) \right]\\
        &\subseteq \bigcup_{A \in \mathcal{A}}  \left[ \partial\left( \bigcap_{s \in A}Y_s \right) \cup \partial \left( \bigcap_{s \notin A}Y_s^c \right) \right]\\
        &\subseteq \bigcup_{A \in \mathcal{A}}  \left[ \left( \bigcup_{s \in A}\partial Y_s \right) \cup \left( \bigcup_{s \notin A} \partial Y_s^c \right) \right]\\
        &= \bigcup_{A \in \mathcal{A}}  \left[ \left( \bigcup_{s \in A}\partial Y_s \right) \cup \left( \bigcup_{s \notin A} \partial Y_s \right) \right]\\
        &= \bigcup_{i=1}^{w_{j-1}} \partial Y_i.
    \end{align*}
    For any nonempty set $A \subseteq \R^n$, the set of breakpoints of $\mathbbm{1}_A$ is $\partial A$, so the above inclusion ends the proof.
\end{proof}

\begin{definition} For any single neuron with a linear threshold activation with $k$ inputs, the output is the indicator of an open halfspace, i.e., $\mathbbm{1}_{\{x\in \R^k : \langle a, x \rangle + b > 0\}}$ for some $a \in \R^k$ and $b\in \R$. We say that $\{x\in \R^k : \langle a, x \rangle + b = 0\}$ is the {\em hyperplane associated with this neuron}. 
\end{definition}

\section{Main Proofs (Linear Threshold DNNs)}\label{sec:2}

\medskip

\noindent \textbf{Proof of Theorem~\ref{theorem:1}}
 


\medskip

\begin{restatable}{proposition}{propositionone}
\label{proposition:1}
${\textup{LT}_1(1)} = \textup{PWC}_1$, i.e., linear threshold neural networks with a single hidden layer can compute any piecewise constant function $\mathbb{R}\rightarrow \mathbb{R}$. {Furthermore, given that $f \in \textup{PWC}_1 $, it follows that}
\begin{equation*}
  {\min_{N \in \mathcal{N}(\textup{LT}_1(1),f)}|N| \leq 3|\mathcal{P}_f^*|+1.}
\end{equation*}
\end{restatable}
\begin{proof}
Let $f: \mathbb{R}\rightarrow \mathbb{R}$ a piecewise constant function. By definition, {the union of polyhedra in $\mathcal{P}_f^*$ is $\R$} 
and such that $f$ is constant on the relative interior of each of the polyhedra. In $\mathbb{R}$, non empty polyhedra are either reduced to a point, or they are the intervals of the form $[a,b]$, $]-\infty, a]$ , $[a, +\infty[$  with $a \leq  b \in \mathbb{R}$, or $\mathbb{R}$ itself. We first show that we can compute the indicator function on each of the interior of those intervals with at most two neurons. The interior of $[a, +\infty[$, $]-\infty, b]$ or $\mathbb{R}$ can obviously be computed by one neuron (e.g. $x \mapsto \mathbbm{1}_{\{ax < 0\}}$ with $a=0$ for  $\mathbb{R}$). The last cases (singletons and polyhedron of the form $[a,b]$) requires a more elaborate construction. To compute the function $\mathbbm{1}_{\{x \in ]a,b[ \}}  $, it is sufficient to implement a Dirac function, since $\mathbbm{1}_{\{x \in ]a,b[\}} =  \mathbbm{1}_{\{ x < b \}} - \mathbbm{1}_{\{ x < a\}} - \delta_a(x) $ where $\delta_{a}$ is the Dirac in $a \in \mathbb{R}$, i.e, $\delta_a: \mathbb{R} \rightarrow \mathbb{R}, \; x \mapsto \mathbbm{1}_{\{ x = a\}}$.  
  $\delta_a$ can be computed by a linear combination of three neurons, since $g_a: \mathbb{R} \rightarrow \mathbb{R}, \; x \mapsto  \mathbbm{1}_{ \mathbb{R} } - (\mathbbm{1}_{\{ x < a\}} + \mathbbm{1}_{\{x> a \} }) $ is equal to $\delta_a$. Using a linear combination of the basis functions (polyhedra and faces), we can compute exactly $f$. To show that {$3|\mathcal{P}_f^*|+1$} neurons suffice, $\mathbbm{1}_{\mathbb{R}}$ is computed with one shared neuron, and then $3$ other neurons are needed at most for one polyhedron using our construction.   
\end{proof}



We next show that starting with two dimensions, linear threshold NNs with a single hidden layer cannot compute every possible piecewise constant function. 

\begin{proposition}\label{proposition:2}
Let $C_2:=\{(x_1, x_2)\in \mathbb{R}^{2} \; | \;   0 \leq x_1, x_2 \leq 1\}$. Then $\mathbbm{1}_{C_2}$ cannot be represented by any linear threshold neural network with one hidden layer.
\end{proposition}
\begin{proof}
Consider any piecewise constant function on $\R^2$ represented by a single hidden layer neural network. Let $g:= x \mapsto \sum_{i=1}^{m}\alpha_i \mathbbm{1}_{\{x \in \R^2\;:\; \langle a_i, x \rangle + b_i< 0\}}$ with $\alpha_1, \cdots, \alpha_m \in \mathbb{R}$, $a_1, \ldots, a_m \in \R^2$ and $b_1, \ldots, b_m \in \R$, be a single hidden layer neural network with the smallest size. This implies that if $i\neq j$ then the halfspace $\{x \in \R^2: \langle a_i, x\rangle + b_i > 0\}$ is different from the halfspace $\{x \in \R^2: \langle a_j, x\rangle + b_j > 0\}$. Otherwise, we may either replace the two corresponding neurons with a single neuron with weight $\alpha = \alpha_i + \alpha_j$ and we would obtain a smaller neural network. 
This implies that the set of breakpoints for $g$ is a union of lines in $\R^2$. However, the set of breakpoints $\mathbbm{1}_{C_2}$ is formed by the sides of the cube, which is a union of finite length line segments. This shows that $\mathbbm{1}_{C_2}$ cannot be represented by a single hidden layer linear threshold NN.
\end{proof}

{ In the two following Lemmata, we assemble the last pieces towards a complete proof of Theorem~\ref{theorem:1} which states that $2$ hidden layers actually suffice to represent any piecewise constant function in $\textup{PWC}_n$. }

\begin{lemma}\label{lemma:2}
Let $P$ be a polyhedron in $\mathbb{R}^{n}$ given by the intersection of $m$ halfspaces. 
{Then, $\mathbbm{1}_P \in \textup{LT}_n(2)$ and }
\begin{equation*}
  {\min_{N \in \mathcal{N}(\textup{LT}_n(2),\mathbbm{1}_P)} |N| \leq m+1.}
\end{equation*}
\end{lemma}
\begin{proof}
Let $P$ a polyhedron, i.e. $P= \{x \in \mathbb{R}^{n} \;| \; A x \leq b \}$ with $A =(a_1, \cdots, a_m)^{\top}\in \mathbb{R}^{m\times n}$ and $b=(b_1, \cdots, b_m) \in \mathbb{R}^{m}$. Let us consider the $m$ neurons $ (\phi_i: x \mapsto \mathbbm{1}_{\{x \in \R^n:\;\langle a_i, x \rangle > b_i\}} )_{1\leq i \leq m}$, and $\phi: x \mapsto \sum_{i} \phi_i (x)$. Then for all $x \in \mathbb{R}^{n}$, $\phi(x) < 1$ if and only if $x \in P$. Now, defining $\psi : y \mapsto \mathbbm{1}_{\{y \in \R:\; y < 1\}} $ yields $\psi \circ \phi = \mathbbm{1}_{P}$. $\psi$ can obviously be computed with a neuron. Therefore, one can compute $\mathbbm{1}_P$ with $m$ neurons in the first hidden layer and one neuron in the second, which {provides a construction of $\mathbbm{1}_P$ with $m+1$ neurons in total, and} proves the result.
\end{proof}


\begin{lemma}\label{lemma:3}
  Let $P$ be a polyhedron in $\mathbb{R}^{n}$.  Then {$\mathbbm{1}_{\mathring{P}}$} can be computed with a two hidden layer { linear threshold NN}, using the indicator of $P$ and the indicators of its faces. 


\end{lemma}
\begin{proof}
Let $P$ be a polyhedron. First, we always have $\mathbbm{1}_{\mathring{P}} = \mathbbm{1}_{P} - \mathbbm{1}_{ \text{Union of facets of P}} $. Therefore it is sufficient to prove that we can implement $\mathbbm{1}_{\text{Union of facets of P}}$ for any $P$. Using the inclusion exclusion principle on indicator functions, suppose that the facets of $P$ are $f_1, \cdots, f_l$, then:
$$\mathbbm{1}_{\bigcup_{j=1}^{l}f_j} = \sum_{j=1}^{l} (-1)^{j+1} \sum_{1 \leq i_1 < \cdots < i_j \leq l } \mathbbm{1}_{f_{i_1} \cap \cdots \cap f_{i_j} }{.}$$

It should be noted that  for any $j \in \{1, \cdots, l\}, \;  F = f_{i_1} \cap \cdots \cap f_{i_j}$ is either empty, or a face of $P$, hence a polyhedron of dimension lower or equal to $\dim(P)-1 $. Therefore, using Lemma \ref{lemma:2}, we can implement $F$ with a two hidden neural network with at most $m+1$ neurons, where $m$ is the number of halfspaces in an inequality description of $P$. If $s$ is the number of faces of $P$, then there are at most $s$ polyhedra to compute.
\end{proof}

\begin{proof}[Proof of Theorem~\ref{theorem:1}]
By definition, in order to represent $f\in \text{PWC}_n$ it suffices to compute the indicator function of the relative interior of each polyhedron in one of its { smallest} polyhedral complex {$\mathcal{P}_f^*$}. This can be achieved with just two hidden layers using Lemma \ref{lemma:3}. This establishes the equalities in the statement of the theorem. The strict containment $\text{LT}_n(1) \subsetneq \text{LT}_n(2)$ is given by Proposition~\ref{proposition:2}.


Let $m$ be the total number of halfspaces used in an inequality description of all the polyhedra in the polyhedral complex {$\mathcal{P}_f^*$}. Since all faces are included in the polyhedral complex, there exists an inequality description with {$m \leq 2|\mathcal{P}_f^*|$}. The factor 2 appears because for each facet of a full-dimensional polyhedron in {$\mathcal{P}_f^*$}, one may need both directions of the inequality that describes this facet. Then the construction in the proofs of Lemmas~\ref{lemma:2} and~\ref{lemma:3} show that one needs at most {$m \leq 2|\mathcal{P}_f^*|$} neurons in the first hidden layer and at most {$|\mathcal{P}_f^*|$} neurons in the second hidden layer.\end{proof}



\noindent \textbf{Proof of Proposition~\ref{PROPOSITION:LOWER_BOUND_CLASS}}

\begin{proof}[Proof of Proposition~\ref{PROPOSITION:LOWER_BOUND_CLASS}]
Consider the sets $P_1 := \{x \in \mathbb{R}^{n}: x_1 \leq 0 \}$,
 $P_i := \{x \in \mathbb{R}^{n}: (i-2) < x_1 \leq i-1\}$ for $i \in \{2, \cdots, m-1\}$, and $P_m := \{x \in \mathbb{R}^{n}: x_1 > m-2  \}$. Note that $\bigcup_{i=1}^{m}P_i=\mathbb{R}^{n}$. {Let $f \in \text{PWC}_{n}$ be such that $f(x) = i$ for all $x \in P_i$, 
 where $i \in \{1,\dots,m\}$. It is easy to see that $\mathcal{P}_f^* = \{P_1,\dots,P_m\}$}, and that the breakpoints of $f$ is a set of $m-1$ hyperplanes, with empty pairwise intersections. 
 {By Theorem~\ref{theorem:1}, $f \in \textup{LT}_n(2)$, and according to Corollary~\ref{lemma:breakpointsbetweenlayers}, any neural network $N \in \mathcal{N}(\textup{LT}_n(2),f)$
 must have these hyperplanes associated with neurons in the first hidden layer, necessitating at least $m-1$ neurons in this layer. Taking into account the neurons in the subsequent layer, we establish that $|N|\geq m = |\mathcal{P}_f^*|$.}
\end{proof}

{The proofs of Proposition~\ref{th:upperboundSLTn} and Theorem~\ref{th:a.e.SLTn} rely on some facts from polyhedral geometry, which are incorporated into the subsequent lemma.

\begin{lemma} \label{lemma:BoundForPolyhedraComplex}
    Let $\mathcal{P}$ be a finite polyhedral complex in $\R^n$ with $\full(\mathcal{P}) = \{P_1,\dots,P_m\}$, where $m \in \N_+$. If the union of all polyhedra in $\mathcal{P}$ equals $\R^n$, then the following statements are all true.
    \begin{enumerate}
        \item $\bigcup_{i=1}^m P_i = \R^n$.
        \item For any $k$ dimensional polyhedron $F \in \mathcal{P}$ with $0 \leq k \leq n$, there exist $n-k+1$ distinct full-dimensional polyhedra in the complex whose common intersection equals $F$.
        \item $m \leq \vert \mathcal{P} \vert < \left( \frac{e m}{n+1} \right)^{n+1},$ where $e\approx 2.71828$ is Euler's number.
    \end{enumerate}
\end{lemma}

\begin{proof}


Suppose $\bigcup_{i=1}^m P_i \neq \R^n$, and consider $x \in \R^n\backslash \left( \bigcup_{i=1}^m P_i \right)$, then there exists some $\varepsilon > 0$ such that $\mathcal{B}(x,\varepsilon) \subseteq \R^n\backslash \left( \bigcup_{i=1}^m P_i \right)$ since $\bigcup_{i=1}^m P_i$ is closed as it is a finite union of polyhedra. This leads to a contradiction since $\mathcal{P}$ covers $\R^n$ but a finite union of polyhedra with dimension at most $n-1$ cannot cover $\mathcal{B}(x,\varepsilon)$. This proves part 1.


For part 2., we first observe that $F \in \mathcal{P}$ if and only if $F$ is a face of some full-dimensional polyhedra in $\mathcal{P}$. One direction follows from the definition of a polyhedral complex. For the other direction, consider any $F \in \mathcal{P}$. Using part 1., 
$$F = \R^n \cap F = \left(\bigcup_{i=1}^m P_i\right) \cap F = \bigcup_{i=1}^m \left(P_i \cap F\right). $$
By definition of a polyhedral complex, $P_i \cap F$ is a face of $F$, $\forall i \in \{1, \ldots, m\}$. The above equality thus implies that $F$ is a finite union of some faces of $F$. This implies that one of these faces must be $F$ itself, i.e., there exists $i \in \{1, \ldots, m\}$ such that $P_i \cap F = F$. Also, by definition $F = P_i \cap F$ is a face of $P_i$, which proves that $F$ is a face of some full-dimensional polyhedra in $\mathcal{P}$.


Next consider any $k$ dimensional polyhedron $F \in \mathcal{P}$. By the argument above, there exists $i \in \{1, \ldots, m\}$ such that $F$ is a face of $P_i$. When $k=n$, the result is trivial with $F = P_i$. We now show the result for $k=n-1$. Let $\langle a, x \rangle \leq b$ be a facet defining inequality for $P_i$ corresponding to $F$. Let $x_0$ be a point in the relative interior of $F$. Consider the sequence $x_0 + \frac{1}{n}a$ and observe that no point in this sequence is contained in $P_i$. Since this is an infinite sequence, there must exist $j\in \{1, \ldots, m\}$ with $j\neq i$ such that $P_j$ contains infinitely many points from this sequence. Taking limits over this subsequence and using the fact that $P_j$ is closed, we obtain that $x_0 \in P_j$. Thus, $x_0 \in P_i \cap P_j$ and $P_i \cap P_j$ is a common face of $P_i$ and $P_j$. However, since $x_0$ is in the relative interior of the facet $F$, this common face must be $F$. Thus we are done for the case $k=n-1$. For any $k \leq n-2$, the face $F$ must be the intersection of $n-k$ distinct facets of $P_i$. By the argument above, each of these $n-k$ facets is given by the intersection of $P_i$ with another full-dimensional polyhedron in the complex. Since these are distinct facets, the corresponding full-dimensional polyhedra must be distinct. Including $P_i$, the intersection of these $n-k+1$ polyhedra equals the intersection of these $n-k$ facets of $P_i$, which is precisely $F$. This finishes the proof of part 2.

The first inequality of 3. follows from the fact that $P_1, \ldots,  P_m \in \mathcal{P}$. From 2., every polyhedron in the complex of dimension $k$ must be the intersection of $n-k+1$ distinct full-dimensional polyhedra. 
Therefore, ${m \choose n-k+1}$ gives an upper bound for the number of all the $k$ dimensional polyhedra in $\mathcal{P}$. Now we can give an upper bound for $\vert \mathcal{P} \vert$:

$$\vert \mathcal{P} \vert \leq \sum_{k=0}^n { m \choose n-k+1} < \left( \frac{e m}{n+1} \right)^{n+1},$$
where the second inequality comes from using Stirling's approximation.
\end{proof}}

\noindent \textbf{Proofs of {Proposition~\ref{th:upperboundSLTn}} and Theorem~\ref{th:a.e.SLTn}}

\begin{proof}[Proof of {Proposition~\ref{th:upperboundSLTn}}]
    {By definition, the smallest polyhedra complex compatible with $f$, $\mathcal{P}_f^*$, has $p$ full-dimensional polyhedra. Then the} construction in Theorem~\ref{theorem:1} implies $f$ can be {computed} by a linear threshold NN with {size at most $3 \vert \mathcal{P}_f^* \vert \leq 3\left( \frac{e p}{n+1} \right)^{n+1}$}, where the inequality holds by part 3 of Lemma~\ref{lemma:BoundForPolyhedraComplex}.
    

    
\end{proof}


\begin{proof}[Proof of Theorem~\ref{th:a.e.SLTn}] {Let $\full(\mathcal{P}_f^*) = \{P_1,\dots,P_p\}$. Define $\alpha_i$ as the value of $f$ within the interior of $P_i$ for $i \in \{1,\dots,p\}$. Part 2 of Lemma~\ref{lemma:BoundForPolyhedraComplex} implies that there are at most ${p \choose 2}$ polyhedra of dimension $n-1$ in $\mathcal{P}_f^*$. Consequently, the first hidden layer requires no more than ${p \choose 2}$ neurons to associate the corresponding hyperplanes, along with an additional neuron for computing $\mathbbm{1}_{\R^n}$ to reverse the halfspaces. In the second hidden layer, we employ $p$ neurons to compute the functions $\mathbbm{1}_{\tilde{P}_1},\dots,\mathbbm{1}_{\tilde{P}_p}$, satisfying $\mu(\{x:\mathbbm{1}_{\tilde{P}_i}(x) \neq \mathbbm{1}_{P_i}(x)\}) = 0$ for $i \in \{1, \ldots, p\}$, and the corresponding weights after the second hidden layer are set to $\alpha_1,\dots,\alpha_p$. This construction yields a linear threshold NN of size no more than ${p \choose 2}+1+p=\frac{p(p+1)}{2}+1$, computing a function that is consistent with $f$ within the interior of each $P_i$, and thus equals to $f$ almost everywhere.}

\end{proof}

\medskip

\noindent{\textbf{Proof of Theorem~\ref{th:more_general_complexity}}}

\medskip 

\noindent Consider a neural network with $k$ hidden layers and widths $w=(w_1, \ldots, w_k)$ that implements a function in $\LTw$. The output of any neuron on these data points is in $\{0,1\}$ and thus each neuron can be thought of as picking out a subset of the set $X:= \{x_1, \ldots, x_D\}$. Lemma~\ref{lemma:4} provides a way to enumerate these subsets of $X$ in a systematic manner.

\begin{definition}
For any finite subset $F \subseteq \R^n$, a subset $F'$ of $F$ is said to be {\em linearly separable} if there exists $a\in \R^n$, $b\in \R$ such that $F' = \{x\in F: \langle a, x \rangle + b > 0\}$. 
\end{definition}

The following is a well-known result in combinatorial geometry~\cite{matousek2013lectures}.

\begin{theorem}\label{th:hyp-part-count}
For any finite subset $F \subseteq \R^n$, there are at most $2{|F|\choose n}$ linearly separable subsets.
\end{theorem}

By considering the natural mapping between subsets of $\{1, \ldots, m\}$ and $\{0,1\}^m$, we also obtain the following corollary.

\begin{corollary}\label{th:lin-sep-count}
For any $m \geq 1$, there are at most $2{2^{m}\choose m}$ linearly separable collections of subsets of $\{1, \ldots, m\}$. In other words, $|\mathcal{L}_m| \leq 2{2^{m}\choose m}$.
\end{corollary}

\begin{algorithm}
	\small
	\caption{Algorithm to solve~\eqref{eq:problem_3_layer_HDNN} for linear threshold NNs with $n$ inputs, $k$ hidden layers and widths $w = (w_1, \ldots, w_k)$. }
	\label{Algorithm:heavisideopt}
	\begin{algorithmic}[1]
	  \STATE \textbf{Input} Dimension $n$, Dataset $(x_i, y_i)_{i=1}^D$, Integers $w_1, \ldots, w_k$ \;
	  
	   \STATE \textbf{Output} Solution of Problem \ref{eq:problem_3_layer_HDNN} \;
	
	\STATE Define $X = (x_1, \ldots, x_D) \subseteq \R^n$. Let $\mathcal{H}$ be the collection of linearly separable subsets of $X$.
	
	\STATE Initialize $OPT = +\infty$, $SOL = \emptyset$.
	 
    \FOR{each choice of $H_1, \ldots, H_{w_1} \in \mathcal{H}$, $\mathcal{A}^i_{1}, \ldots, \mathcal{A}^i_{w_i} \in \mathcal{L}_{w_{i-1}}$ for $i=2, \ldots, k$}
    
    \STATE Define $Y^1_j$ to be any halfspace of $\R^n$ such that $X \cap Y^1_j=H_j$ for $j=1, \ldots, w_1$.
    
      \STATE Set the weights of the neurons in the first layer to compute $Y^1_j$ for $j=1, \ldots, w_1$.
    
    \FOR{i = 2 to k}
    \FOR{j=1 to $w_i$}
    
    \STATE   Define $Y^i_j = \bigcup_{A \in \mathcal{A}^i_j} \left[ (\, \bigcap_{s \in A}Y^{i-1}_s \, ) \cap \, ( \bigcap_{s \notin A}(Y^{i-1}_s)^{c}  \, ) \right]$.
    
    \STATE Set the weights of neuron $j$ of layer $i$ in accordance with $\mathcal{A}^{i}_{j}$ to compute $Y^i_j $.
    
    \ENDFOR
    \ENDFOR
    
    \STATE For each $i=1, \ldots, D$ and $j=1, \ldots, 
    w_k$,  compute $\delta_{ij} \leftarrow \mathbbm{1}_{Y^{k}_{j}}(x_i)$, using the neural network constructed so far.

    \STATE Solve the convex minimization problem in the decision variables $\gamma_1, \ldots, \gamma_{w_k}\in \R$:
    
     $$\textrm{temp} = \min_{\gamma \in \mathbb{R}^{w_k}} \;\;\sum_{i=1}^D \; \ell \left( \sum_{j=1}^{w_k} \gamma_j \delta_{ij}, y_i \right ){.}$$
    
    \STATE If $\textrm{temp} < OPT$, then update $OPT = \textrm{temp}$ and $SOL$ to be the current neural network with weights computed in the previous steps. 
    \ENDFOR

	\end{algorithmic}
\end{algorithm}

\begin{proof}[Proof of Theorem~\ref{th:more_general_complexity}]
Algorithm~\ref{Algorithm:heavisideopt} solves~\eqref{eq:problem_3_layer_HDNN}. The correctness comes from the observation that a recursive application of Lemma~\ref{lemma:4} shows that the sets $Y^k_1, \ldots, Y^k_{w_k}$ computed by the algorithm, intersected with $X$ are all possible subsets of $X$ computed by the neurons in the last hidden layer. The $\gamma_1, \ldots, \gamma_{w_k}$ are simply the weights of the last layer that combine the indicator functions of these subsets to yield the function value of the neural network on each data point. The convex minimization problem in line 13 finds the optimal $\gamma_j$ values, for this particular choice of subsets $Y^k_1, \ldots, Y^k_{w_k}$. Selecting the minimum over all these choices solves the problem.

The outermost for loop iterates at most  $O(D^{w_1 n}\cdot 2^{\sum_{i=1}^{k-1}{w_i^{2} w_{i+1}}})$
times using Theorem~\ref{th:hyp-part-count} and Corollary~\ref{th:lin-sep-count}. The computation of the $\delta_{ij}$ values in Step 14 can be done in time $\textup{poly}(D, w_1, \ldots, w_k)$. The convex minimization problem in $w_k$ variables and $D$ terms in the sum can be solved in $\textup{poly}(D,w_k)$ time. Putting these together gives the overall running time.
\end{proof}

We now show that the exponential dependence on the dimension $n$ in Theorem \ref{th:more_general_complexity} is actually necessary unless P=NP.
We consider the version of~\eqref{eq:problem_3_layer_HDNN} with single neuron and show that it is NP-hard with a direct reduction.



\begin{theorem}\label{theorem:OneNodeLearning} (NP-hardness). The One-Node-Linear-Threshold problem, i.e., Problem \ref{eq:problem_3_layer_HDNN} with $k=1$ and $w_1 = 1$,  is NP-hard when the dimension $n$ is considered part of the input. This implies in particular that Problem \ref{eq:problem_3_layer_HDNN} is NP-hard when $n$ is part of the input.
\end{theorem}
\begin{proof} We here use a result of \cite[Theorem 3.1]{hoffgen1995robust}, which showed that the following decision problem is NP-complete.
~\\

\noindent\textit{MinDis(Halfspaces):
Given disjoint sets of positive and negative examples of $\mathbb{Z}^{n}$ and a bound $k\geq 1$, does there exist a separating hyperplane which leads to at most $k$ misclassifications?}
~\\

MinDis(Halfspaces) is a special case of~\eqref{eq:problem_3_layer_HDNN} with a single neuron: given data points $x_1, \cdots, x_D \in \mathbb{R}^{n}$ and $y_1, \cdots, y_D \in \{0,1\}$, compute $\alpha \in \mathbb{R}^{n}, \beta \in \mathbb{R}$ that minimizes $\frac{1}{D} \sum_{i=1}^{D} (\mathbbm{1}_{\{\langle \alpha, x_i\rangle + \beta > 0\}} - y_i)^{2}$.
\end{proof}





\section{Shortcut linear threshold {NNs}} \label{sec:ShortcutNN}
\subsection{Representability of shortcut linear threshold {NNs}}


\begin{proof}[Proof of Theorem~\ref{th:SLTnrepresentability}]
{Arora et al. \cite{arora2018understanding} proved that $\ReLU_n = \CPWL_n$, and it's clear that $\CPWL_n\subseteq \PWL_n$ and $\SLT_n(2)\subseteq \PWL_n$, so it remains to prove that $\PWL_n \subseteq \SLT_n(2)$. Let $f \in \PWL_n$ be an arbitrary piecewise linear function, and let $\mathcal{P}_f^* = \{P_1,\dots,P_m\}$. By definition, $f(x) = \sum_{i=1}^{m}(a_i^\top x + c_i)\mathbbm{1}_{\mathring{P}_i}(x)$ for some $a_i \in \R^n,\ c_i \in \R$, where $i \in \{1,\dots,m\}$. By the proof of Lemma~\ref{lemma:3}, we are able to compute $\mathbbm{1}_{\mathring{P_i}}$ by a linear combination of the outputs of some neurons in the second hidden layer. In other words, let $x^{(2)}= [\mathbbm{1}_{X_1^1}(x),\dots,\mathbbm{1}_{X_{\ell_1}^{1}}(x),\dots,\mathbbm{1}_{X^{m}_{\ell_{m}}}(x)]^\top$ be the output of the second hidden layer such that for every $i\in \{1, \ldots, m\}$, we have $\mathbbm{1}_{\mathring{P_i}}(x) = \sum_{j=1}^{\ell_i} \alpha_j^{(i)} \mathbbm{1}_{X_j^i}(x)$, where $\ell_i \in \N_+,\ \alpha_j^{(i)} \in \R$, and $\mathbbm{1}_{X^i_j}$ are computed by the individual neurons in the second hidden layer. Note that the number of neurons in the second hidden layer is $w_2 = \sum_{k=1}^{m}\ell_k$. 

Now consider introducing a shortcut connection with 
$$A = [\alpha_{1}^{(1)}a_1,\dots,\alpha_{\ell_1}^{(1)}a_1,\dots,\alpha_{\ell_m}^{(m)}a_m] \in \R^{n \times w_2}$$ 
and $b = [\alpha_1^{(1)} c_1,\dots,\alpha_{\ell_1}^{(1)}c_1,\dots,\alpha_{\ell_m}^{(m)}c_m]^\top \in \R^{w_2}$, then the output of this shortcut NN is given by:
\begin{align*}
    \langle A^{\top} x + b,x^{(2)}\rangle &= 
    \left\langle
    \begin{bmatrix}\alpha_{1}^{(1)}(a_1^\top x+c_1)\\\vdots\\\alpha_{\ell_1}^{(1)}(a_1^\top x+c_1)\\\vdots\\\alpha_{\ell_m}^{(m)}(a_m^\top x+c_m)\end{bmatrix},
    \begin{bmatrix}\mathbbm{1}_{X_1^1}(x)\\\vdots\\\mathbbm{1}_{X_{\ell_1}^{1}}(x)\\\vdots\\\mathbbm{1}_{X^{m}_{\ell_{m}}}(x)\end{bmatrix}
    \right\rangle\\
    &= \sum_{i=1}^m \sum_{j=1}^{l_i}(a_i^\top x+c_i) \alpha_j^{(i)} \mathbbm{1}_{X^i_j}(x)\\
    &= \sum_{i=1}^m (a_i^\top x+c_i) \sum_{j=1}^{l_i} \alpha_j^{(i)} \mathbbm{1}_{X^i_j}(x)\\
    &= \sum_{i=1}^{m} \left(a_i^\top x + c_i \right) \cdot \mathbbm{1}_{\mathring{P}_{i}}(x) \\
    &= f(x).
\end{align*}
This establishes that $\PWL_n \subseteq \SLT_n(2)$, completing the proof. 
}

\end{proof}

\subsection{Adapting the ERM algorithm for shortcut linear threshold NNs}\label{sec:ERMforSLT}
{
We now consider solving the ERM problem for a $\R^n \rightarrow \R$ shortcut LT NN with $k$ hidden layers, and $w = (w_1,\dots,w_k)$. Upon comparing with Algorithm~\ref{Algorithm:heavisideopt}, we note that the difference between our shortcut linear threshold NNs and the linear threshold NNs solely resides in the presence of a shortcut connection in the former across the piecewise regions. Except for the output layer, all other layers in the two networks can be analogously compared. Consequently, the algorithmic process concerning linear threshold NNs can be seamlessly incorporated, except for those steps involving the output layer. Hence, in the global optimization algorithm, the only difference arises in Step 15:}
$$\textrm{temp} = \min_{\gamma\in \mathbb{R}^{w_k},a_j\in \R^{n}\;\forall j\in[w_k]} \;\;\sum_{i=1}^D \; \ell \left( \sum_{j=1}^{w_k} (a_j^\top x_i + \gamma_j) \delta_{ij}, y_i \right),$$
which can be resolved in poly$(D,(n+1)w_k)$ time.

{\section{Discussions and open questions}\label{sec:5}
\subsection{Linear threshold NNs}\label{sec:5.1.1}}

{
Results from Boolean circuit complexity can be used to show that our general construction in Theorem~\ref{theorem:1} may produce 2 hidden layer networks that are exponentially larger than networks that use 3 hidden layers. 

\begin{example}\label{Continuous parity}
Consider the piecewise constant function $p_n(x): \R^n \rightarrow \R$ defined as $$p_n(x) = \sigma\left(\prod_{i=1}^n x_i\right){,}$$ where $\sigma$ is the threshold activation function. $p_n$ has $\Omega(2^n)$ pieces implying that the construction from Theorem~\ref{theorem:1} has size $\Omega(2^n)$. However, $p_n$ can be represented by a linear threshold NN with $3$ hidden layers and $\mathcal{O}(n)$ size.
\end{example}

\begin{example}[Braid arrangement \cite{wachs2006poset}]\label{Braid arrangement} 
Consider a $\R^n \rightarrow \R$ function 
$$B_n(x) = \sigma \left(\prod_{1\leq i < j \leq n} (x_j-x_i) \right) {,}$$ where $\sigma$ is the threshold activation function. $B_n$ has $\Omega(n!)$ pieces implying that the construction from Theorem~\ref{theorem:1} has size $\Omega(n!)$. However, $B_n$ can be represented by a linear threshold NN with $3$ hidden layers and $\mathcal{O}(n^2)$ size.
\end{example}
The {\em parity} function is defined as the function $\text{par}: \{0,1\}^n \to \{0,1\}$ as $\text{par}(x) = \sum_{i=1}^n x_i \; (\textrm{mod} \;\;2)$. It is well-known that the parity function can be implemented by a linear threshold NN with 2 hidden layers and $\mathcal{O}(n)$ nodes, when restricted to 0/1 inputs~\cite{muroga1971threshold,paturi1994approximating}. Observe that $p_n(x) = \text{par}(\sigma(x))$ where we apply the threshold activation $\sigma$ component wise on $x\in \R^n$. This proves that $p_n$ can be computed by a linear threshold NN with 3 hidden layers and $\mathcal{O}(n)$ size. Each orthant of $\R^n$ is a piece of $p_n$ since any adjacent orthant has a different value. Similarly, if we define $\text{diff}(x) \in \R^{n(n-1)/2}$ by $\text{diff}(x)_{ij} = x_i - x_j$ for $1 \leq i<j \leq n$, then $B_n(x) = \text{par}(\sigma(\text{diff}(x)))$. The fact that $B_n$ has $n!$ pieces comes from results on the so-called {\em Braid arrangement}~\cite{wachs2006poset}. 

Lower bounds for the number of {\em wires} used in a linear threshold NN have also been studied in the Boolean circuit complexity literature~\cite{impagliazzo1997size,paturi1994approximating,kane2016super}. The number of wires is the number of connections between neurons when the neural network is viewed as a directed acyclic graph. This amounts to the number of nonzero entries of the matrices involved in the affine transformations in Definition~\ref{def:DNN}. Tight bounds that are superlinear but subquadratic in $n$ are known for the wire complexity of any constant depth linear threshold NN implementing the parity function. These results also imply that there is a $\mathcal{O}(\log\log n)$ depth NN that implements the parity function with $\mathcal{O}(n)$ wires. See~\cite{impagliazzo1997size,paturi1994approximating} for details. These constructions can be used to implement $p_n$ and $B_n$. 

It is not clear if the functions in Example~\ref{Continuous parity} and Example~\ref{Braid arrangement} can be implemented by 2 hidden layers networks of polynomial size, or whether there exist superpolynomial lower bounds on the size of such networks. In the first case, we will know that our construction in this paper is suboptimal. In the second case, we will have our gap result for 2 versus 3 hidden layers.}

{ \subsection{Shortcut linear threshold NNs} \label{sec:5.1.2}

Shortcut linear threshold NNs (SLT NNs) may bear a superficial resemblance to Residual Neural Networks (ResNet), mainly due to the incorporation of shortcut or skip connections in both architectures. However, ResNet, a significant advancement in deep learning pioneered by He et al. \cite{he2016deep}, employs skip connections to enable a straightforward addition of skipped layers. This design strategy aims to combat issues such as the vanishing gradient problem prevalent in deep networks. In contrast, SLT NNs use a similar shortcut concept but apply it differently, producing output through the dot product of the input layer (after linear transformation) and the final hidden layer. This distinction signifies a shift from ResNet's engineering focus to a more theoretical perspective in SLT NNs, aiming to augment the representability of neural networks by transitioning from piecewise constant to piecewise linear function representation.}

{In the realm of machine learning, the concept of VC-dimension, named after Vapnik and Chervonenkis, acts as a measure of the capacity, or complexity, of a hypothesis class, essentially characterizing the sample complexity needed to learn from that hypothesis class. This makes it a fundamental tool in learning theory.
%
%
As part of our motivation for this novel type of shortcut connections, we aim to compare the complexity of SLT NNs and ReLU NNs using this dimensionality measure. This comparison aids in gauging the ability of the networks to learn and generalize from data when using SLT NNs. 
A fundamental result on the VC-dimension of parametrized classes of functions~\cite[Theorem 8.4]{anthony1999neural} can be used to show that the VC-dimension of a SLT NN with $n$ inputs and $k$ hidden layers is $\mathcal{O}\left((W+nw_k)^2\right)$, where $W$ corresponds to the number of learning parameters not including the shortcut connection, $w_k$ represents the neurons in the last hidden layer; $nw_k$ designates the additional parameters associated with the linear transformation $A$ in the shortcut connections (note that the parameters corresponding to the shift $b$ are already included in $W$ because they are present in the original linear threshold NN without the shortcut). For a similarly structured ReLU NN devoid of the shortcut connection, the VC-dimension is $\Theta \left( Wk \log W \right)$ (see \cite{bartlett_nearly-tight_2017}). Therefore, the discrepancy in the VC-dimension between comparable architectures is not dramatic, and the ability of shortcut linear threshold functions to represent discontinuous piecewise linear functions can potentially give them a competitive edge over ReLU NNs.}

{Furthermore, we can construct a globally optimal ERM algorithm for shortcut linear threshold NNs across all architectures, an accomplishment not yet attained for ReLU networks beyond specific restricted structures~\cite{arora2018understanding, goel2017reliably,goel2018learning,goel2019learning,dey2020approximation,boob2022complexity,GKMR21,froese2022computational}.}

{\subsection{Complexity of neural network training}

There has been a recent strand of work around the computational complexity of training neural networks provably to global optimality. It has been known for decades that the complexity of neural network training with classical activation functions is hard, and recently this insight has been extended to ReLU activations as well; see~\cite{abrahamsen2021training,froese2022computational,goel2020tight,bertschinger2022training,dey2020approximation}. On the positive side, fixed-parameter tractable algorithms and approximation algorithms have been designed~\cite{froese2022computational,boob2022complexity,goel2017reliably,goel2018learning,goel2019learning,arora2018understanding}. However, these algorithms are restricted to architectures with a single hidden layer, or with very similar architectures to single hidden layers. As mentioned in the Introduction of this paper, our training algorithm for Linear Threshold NNs and Shortcut Linear Threshold NNs works for any architecture and has running time polynomial in the number of data points, assuming the size of the network and the data dimension are constants. To the best of our knowledge, a training algorithm with global optimality guarantees for general architectures that has fixed parameter tractability has not appeared in the literature, except for an interesting study by Bienstock, Munoz and Pokutta~\cite{bienstock2018principled}. They formulate the training problem as a linear programming problem which solves the problem to $\epsilon$-accuracy in time that is {\em linear} in the number of data points and polynomial in $\frac{1}{\epsilon}$, assuming the input dimension and the network architecture are fixed. This is done via a discretization of the neural network parameter space and input space. For a general convex loss, our algorithm will also have to be content with $\epsilon$-approximate solutions, since this is the best one can do for minimizing general convex functions. However, our running time is polynomial in $\log(1/\epsilon)$, in contrast to $\frac{1}{\epsilon}$. Moreover, for certain loss functions like the $\ell_1$ or $\ell_\infty$, our algorithm will indeed be exact, because the convex optimization problem becomes a linear programming problem, but the algorithm in~\cite{bienstock2018principled} will still need to rely on discretizations, leading to an approximation. On the other hand, the linear dependence of the algorithm in~\cite{bienstock2018principled} on the number of data points is much better than our algorithm. It should be noted that their analysis does not formally apply to the linear threshold activation functions, since $x\mapsto {\mathbbm{1}}_{\{x>0\}}$ is not Lipschitz continuous, which is an assumption needed in their work.}

{\subsection{Open questions}\label{sec:5.2}}

On the structural side, we need a better understanding of the depth and size tradeoff for (shortcut) linear threshold NNs. In particular, can we show it is possible to represent certain functions with 3 more hidden layers using an exponentially smaller number of neurons compared to what is needed with 2 hidden layers? For instance, in the case of ReLU activations, there exist functions such that going from 2 to 3 hidden layers brings an exponential (in the dimension $n$) gain in the size of the neural network~\cite{eldan2016power}. We think it is an interesting open question to determine if such families of functions exist for linear threshold networks.

We also suspect that one does not need to go beyond 3 hidden layers to improve on the size bounds, if one is prepared to ignore zero measure sets. {This conjecture is formulated based on our empirical observations with these neural networks.}


{
\begin{conjecture}
     For every natural number $n\in \N$, there exists $C(n) \in \R_+$ such that for any $f \in \PWC_n$,
    \begin{equation*}
        \min_{N \in \mathcal{N}_\mu(\textup{LT}_n(3),f)}|N| \leq  C(n) \cdot \min_{k \in \N_+}\min_{N_k \in \mathcal{N}_\mu(\textup{LT}_n(k),f)} |N_k| .
    \end{equation*}
\end{conjecture}

A similar conjecture regarding representing the continuous functions for shortcut linear threshold neural networks can be naturally extended in an analogous manner.

\begin{conjecture}
    For every natural number $n\in \N$, there exists $C(n) \in \R_+$ such that for any $f \in \CPWL_n$, 
    \begin{equation*}
        \min_{N \in \mathcal{N}(\textup{SLT}_n(3),f)}|N| \leq  C(n) \cdot \min_{k \in \N_+}\min_{N_k \in \mathcal{N}(\textup{SLT}_n(k),f)} |N_k| .
    \end{equation*}
\end{conjecture}
}

On the algorithmic side, we solve the empirical risk minimization problem to global optimality with running time that is polynomial in the size of the data sample, assuming that the input dimension and the architecture size are fixed constants. The running time is exponential in terms of these parameters (see Theorem~\ref{th:more_general_complexity}). While the exponential dependence on the input dimension cannot be avoided unless $P=NP$ (see Theorem~\ref{theorem:OneNodeLearning}), another interesting open question is to determine if the exponential dependence on the architectural parameters is really needed, or if an algorithm can be designed that has complexity which is polynomial in both the data sample and the architecture parameters. A similar question is also open in the case of ReLU neural networks~\cite{arora2018understanding}.



\section{Acknowledgements}
The first and third authors gratefully acknowledge support from Air Force Office of Scientific Research (AFOSR) grant FA95502010341 and National Science Foundation (NSF) grant CCF2006587.

\bibliographystyle{spmpsci}   
 \bibliography{references}


\end{document}